\documentclass[conference]{IEEEtran}
\IEEEoverridecommandlockouts
\usepackage{cite}
\usepackage{amsmath,amssymb,amsfonts}
\usepackage{algorithmic}
\usepackage{graphicx}
\usepackage{textcomp}
\usepackage{xcolor}
\def\BibTeX{{\rm B\kern-.05em{\sc i\kern-.025em b}\kern-.08em
    T\kern-.1667em\lower.7ex\hbox{E}\kern-.125emX}}

\usepackage{latexsym}
\usepackage{graphicx}
\usepackage{verbatim}
\usepackage{amsmath, amsthm, amssymb}
\usepackage[linesnumbered,ruled,vlined]{algorithm2e}
\usepackage[pagewise]{lineno}
\usepackage{color}
\usepackage{makecell}
\usepackage{multirow}
\usepackage{warpcol}
\usepackage{dcolumn}
\usepackage{subfigure}

\begin{document}

\title{A Note on Graph-Based Nearest Neighbor Search}

\author{Hongya Wang$^\dag$ \ \  Zhizheng Wang$^\dag$ \ \ Wei Wang$^\ddag$ \ \ Yingyuan Xiao$^\S$ \ \ Zeng Zhao$^\dag$ \ \ Kaixiang Yang$^\dag$ \\
$^\dag$ School of Computer Science and Technology, Donghua University, China\\  $\ddag$ School of Computer Science and Engineering, University of New South Wales, Australia \\ $\S$School of Computer Science and Engineering, Tianjin University of Technology, China \\ \{hywang@dhu.edu.cn\}
% For a paper whose authors are all at the same institution,
% omit the following lines up until the closing ``}''.
% Additional authors and addresses can be added with ``\and'',
% just like the second author.

}

\maketitle

\newtheorem{example}{Example}
\newtheorem{observation}{Observation}
\newtheorem{lemma}{Lemma}
\newtheorem{theorem}{Theorem}
\newtheorem{problem}{Problem}
\newtheorem{definition}{Definition}
\newtheorem{corollary}{Corollary}
\newtheorem{conjecture}{Conjecture}

\begin{abstract}
Nearest neighbor search has found numerous applications in machine learning, data mining and massive data processing systems. The past few years have witnessed the popularity of the graph-based nearest neighbor search paradigm because of its superiority over the space-partitioning algorithms. While a lot of empirical studies demonstrate the efficiency of graph-based algorithms, not much attention has been paid to a more fundamental question: why graph-based algorithms work so well in practice? And which data property affects the efficiency and how? In this paper, we try to answer these questions. Our insight is that ``the probability that the neighbors of a point $o$ tends to be neighbors in the $K$NN graph" is a crucial data property for query efficiency. For a given dataset, such a property can be qualitatively measured by \emph{clustering coefficient} of the $K$NN graph.

%Higher the clustering coefficient is, better the $K$NN graph is connected.

To show how clustering coefficient affects the performance, we identify that, instead of the global connectivity, the local connectivity around some given query $q$ has more direct impact on recall. Specifically, we observed that high clustering coefficient makes most of the $k$ nearest neighbors of $q$ sit in a maximum strongly connected component (SCC) in the graph. From the algorithmic point of view, we show that the search procedure is actually composed of two phases - the one outside the maximum SCC and the other one in it, which is different from the widely accepted single or multiple paths search models. We proved that the commonly used graph-based search algorithm is guaranteed to traverse the maximum SCC once visiting any point in it. Our analysis reveals that high clustering coefficient leads to large size of the maximum SCC, and thus provides good answer quality with the help of the two-phase search procedure. Extensive empirical results over a comprehensive collection of datasets validate our findings.

\end{abstract}

%% main text
\section{Introduction}
\label{section:intro}

Nearest neighbor search among database vectors for a query is a key building block to solve problems such as large-scale image search and information retrieval, recommendation, entity resolution, and sequence matching. As database size and vector dimensionality increase, exact nearest neighbor search becomes expensive and often is considered impractical due to the long search latency. To reduce the search cost, approximate nearest neighbor (ANN) search is used, which provides a better tradeoff among accuracy, latency, and memory overhead.

Roughly speaking, the existing ANN methods can be classified into space-partitioning algorithms and graph-based ones\footnote{Please note the categorization is not fixed/unique}. The space-partitioning methods further fall into three categories - the tree-based, production quantization (PQ) and locality sensitive hashing (LSH)\cite{DatarIIM04:p-stable, JegouDS11:product-quantization}. Recent empirical study shows that graph-based ANN search algorithms are more efficient than the space-partitioning methods such as PQ and LSH, and thus have been adopted by many commercial applications in Facebook, Microsoft, Taobao and etc.~\cite{DouzeSJ18:link-and-code, DBLP:conf/iclr/DongIRW20, FuXWC19:nsg}.

While a lot of empirical studies validate the efficiency of graph-based ANN search algorithms, not much attention has been paid to a more fundamental question: why graph-based ANN search algorithms are so efficient? And which data property affect the efficiency and how? Two recent papers analyze the asymptotic performance of graph-based methods for datasets uniformly distributed on a $d$-dimensional Euclidean sphere~\cite{Laarhoven18:graph-time-space-tradeoff, Prokhorenkova19:graph-practice-to-theory}. The worst-case analysis shows that the asymptotic behavior of a greedy graph-based search only matches the optimal hash-based algorithm~\cite{AndoniLRW17:optimal-hash-based-tradeoff}, which is far worse than the practical performance of graph-based algorithms and thus could not answer these questions.

A few conceptual graph models such as Monotonic Search Network Model~\cite{DearholtGK88:msnet}, Delaunay Graph Model~\cite{SebastianK02:metric-based-nn-retrieval, MorozovB18:similarity-graph-for-maximum-inner-product-search} and Navigable Small World Model~\cite{Kleinberg00:navigation-in-a-small-world, MalkovPLK14:navigable-small-world-network} are proposed to inspire the construction of ANN search graphs. As will be discussed in Section~\ref{section:review-of-graph-based-searching}, none of them can explain the success of graph-based algorithms either. Actually, the vast majority (if not all) of practical ANN search graphs uses approximate $K$NN graph as the index structure instead of the conceptual models due to time or space constraints, and thus is fully devoid of the theoretical features provided by these models.

In this paper, we argue that, for a specific dataset, the clustering coefficient~\cite{Watts98:dynamics-of-small-world-network} of its $K$NN graph is an important indicator on how efficiently graph-based algorithms work. The clustering coefficient of $K$NN graph defines the probability of neighbors of a point are also neighbors. Comprehensive experimental results reveal that higher the clustering coefficient is, more efficiently the graph-based algorithms will perform. Since clustering coefficient is data dependent, graph-based algorithms perform rather worse for datasets such as Random with very small clustering coefficient, whereas do well in datasets such as Sift and Audio with greater ones.

We also study how clustering coefficient affects the performance. The analysis of the complex network indicates that large clustering coefficient leads to high global connectivity~\cite{Newman00:networks-an-intro}. Our insight is that, instead of the global connectivity, the local graph structure is more crucial for high ANN search recall. Particularly, we observed that, for datasets with large clustering coefficient, most of the $k$NN of some given query\footnote{This query can be in or not in the dataset} lie in the maximum strongly connected component (SCC) in the subgraph composed of these $k$NN. Moreover, we show that the search procedure actually consists of two phases, the one outside the maximum SCC and the one in it, in contrast to the common wisdom of single or multiple path search models. Then, we proved that the commonly used graph search algorithm is guaranteed to visit all $k$NN in the maximum SCC under a mild condition, which suggests that the size of the maximum SCC determines the answer quality of $k$NN search. This sheds light on the strong positive correlation between the clustering coefficient and the result quality, and thus answers the two aforementioned questions.

To sum up, the main contributions of this paper are:

\begin{itemize}

\item We introduce a new quantitative measure \emph{Clustering Coefficient} of $K$NN graph for the difficulty of graph-based nearest neighbor search. To the best of our knowledge, this is the first measure which could explain the efficiency of this class of important ANN search algorithms.

  \item The conceptual models such as MSNETs and Delaunay graphs claim that NN could be found by walking in a single path. Instead, we found that the search procedure is actually composed of two phases. In the second phase the algorithm traverse the maximum SCC of $k$NN for a query, of which the size is a determining factor for answer quality, i.e., recall.

   \item We proved that the graph-based search algorithm is guaranteed to visit all points in the maximum SCC once entering it. Extensive empirical results over a comprehensive collection of datasets validate our observations.

\end{itemize}

We believe that this note could provide a different perspective on graph-based ANN search methods and might inspire more interesting work along this line.

\section{Graph-Based Nearest Neighbor Search}
\label{section:review-of-graph-based-searching}

%\subsection{Definitions and Notations}
%\label{section:def-and-notation}

\subsection{Graph Construction and Search Algorithms}
\label{section:graph-construction}
In the sequel, we will use nodes, points and vertices interchangeably without ambiguity. A directed graph $G =(V, E)$ consists of a nonempty vertices set $V$ and a set $E$ of edges such that each edge $e \in E$ is assigned to an ordered pair \{$u,v$\} with $u,v \in V$. Most graph-based algorithms build directed graph to navigate the $k$NN search. To the best of our knowledge, the idea of using graphs to process ANN search can be traced back to the Pathfinder project, which was initiated to support efficient search, classification, and clustering in computer vision systems~\cite{DearholtGK88:msnet}. In this project, Dearholt et al. designed and implemented the \emph{monotonic search network} to retrieve the best match of an entity in a collection of geometric objects. Since then, researchers from different communities such as theory, database and pattern recognition explored different ways to construct search graphs, inspired by various graph/network models such as the \emph{relative neighborhood graph}~\cite{AryaM93:approximate-nn-in-fixed-dimensions, Jaromczyk:rng-and-their-relatives}, \emph{Delaunay graph}~\cite{Navarro99:nn-searching-by-spatial-approximation, SebastianK02:metric-based-nn-retrieval, Aurenhammer91:voronoi-diagrams-a-survey}, \emph{KNN graph}~\cite{ParedesC05:knn-graph-search-in-metric-spaces, KGraph-github-project} and \emph{navigable small world network}~\cite{MalkovPLK14:navigable-small-world-network, MalkovY20:hierarchical-navigable-SMG, Kleinberg00:swn-an-algorithmic-perspetive}. Thanks to its appealing practical performance, the graph-based ANN search paradigm has become an active research direction and quite a lot of new approaches are developed recently~\cite{FuXWC19:nsg, LiZSWZL16:DPG, HajebiASZ11:fast-NN-search-with-kNN-graph, AumullerBF20:ann-benchmark, HarwoodD16:fanng, AoyamaSSU11:degree-reduced-NN-graph, Iwasaki18:onng, Iwasaki16:panng, Baranchuk19:graphs-construction-by-reinforcement-learning, BaranchukPSB19:learning-to-route-in-graphs, ZhouTX019:mobius-transformation-for-NN-search-on-graph, DouzeSJ18:link-and-code, WangL12:iterated-NN-graph-search}.

While motivated by different graph/network models, most (if not all) practical graph-based algorithms essentially use approximate $K$NN graphs as the core index structure. A specific algorithm distinguishes itself from the others mainly in the edge selection heuristics, i.e., the way to select a subset of links between any point and its neighbors. Algorithm~\ref{alg:Graph-construction} depicts the general framework of index construction for graph-based ANN search.

\begin{algorithm}[h]
\label{alg:Graph-construction}
\caption{\textsf{The General Graph Construction Framework}($V$, $K$)}
\KwIn{$V$ is the vertex set and $K$ is used to control the graph qualilty}
\KwOut{The Search Graph $G$}
\For {each $v \in V$}{
Find the exact or approximate $K$ nearest neighbors of $v$\;
Choose a subset $C$ of these $K$NNs based on some specific heuristics\;
Add directed or bi-directional connections between $v$ and every vertex in $C$\;  \tcp{update $G$}
}%endfor
\Return $G$
\end{algorithm}

For almost all graph-based methods, the ANN search procedure is based on the same principle as follows. For a query $q$, start at an initial vertex chosen arbitrarily or using some sophisticated selection rule. Moves along an edge to the adjacent vertex with minimum distance to $q$. Repeat this step until the current element $v$ is closer to $q$ than all its neighbors, and then report $v$ as the NN of $q$. We call this the \emph{single path search} model. {Figure~\ref{Fig:graph-search-without-backtracking} illustrates the searching procedure for $q$ in a sample graph with points $o_1$ to $o_6$, where $o_1$ is the starting point and dashed lines indicate the search path. At the end of searching, the NN of q, i.e., $o_6$ is identified.}

\begin{figure}[http]
\center
 \includegraphics[width= 0.6\columnwidth]{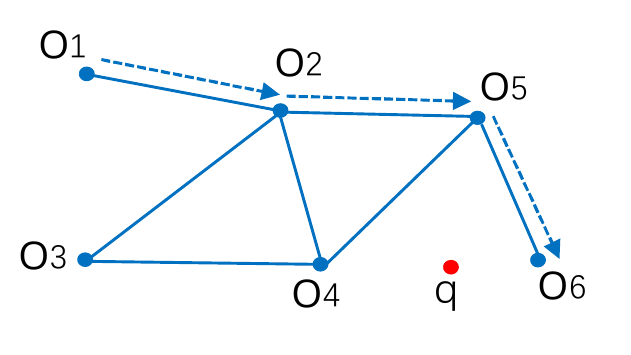}
\caption{\textbf{An illustrative example of graph search}}
\label{Fig:graph-search-without-backtracking}
\end{figure}

For practical ANN search networks, e.g., HNSW and NSG, there is no guarantee that a monotonic search path always exist for any given query~\cite{MalkovY20:hierarchical-navigable-SMG, FuXWC19:nsg}. As a result, it can be easy to get trapped into the local optima, meaning that $v$ is not the NN of $q$. To address this issue, backtracking is solicited -- we need to go back to visited vertices and find another outgoing link to restart the procedure. We call this the \emph{multiple path search} model. Algorithm~\ref{alg:Graph-based-KNN-Search} sketches the commonly adopted search algorithm that allows for backtracking, which will be discussed in detail in Section~\ref{section:the-analysis-of-search-algorithm}. {Figure~\ref{Fig:graph-search-with-backtracking} illustrates a search path with backtracking. The starting point is $o_1$ and $o_2$ is the local optima since the true NN of $q$ is $o_4$. By backtracking the algorithm gets back to $o_3$, which is further to $q$ than $o_2$ and finally find the true NN of $q$.}

\begin{figure}[http]
\center
 \includegraphics[width= 0.6\columnwidth]{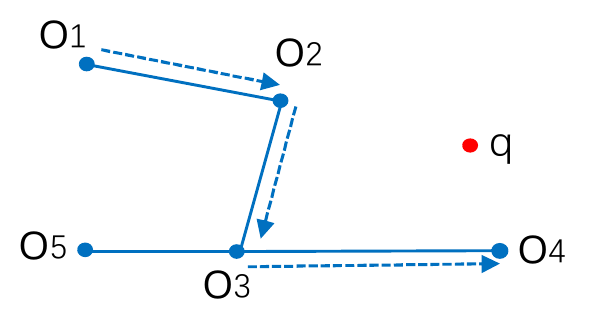}
\caption{\textbf{An illustrative example of graph search with backtracking}}
\label{Fig:graph-search-with-backtracking}
\end{figure}

\begin{algorithm}[h]
\label{alg:Graph-based-KNN-Search}
\caption{\textsf{Graph-based $k$NN Search}($G$, $s$, $q$, $L$)}
\KwIn{Graph $G$, entry vertex $s$, query $q$, priority queue of size $L$}
\KwOut{$k$ nearest neighbors of $q$}
$cand$.push($s$); \tcp{add $s$ to the priority queue of candidates}
$result_L$.push($s$); \tcp{add $s$ to the priority queue that stores $L$ nearest points to $q$}
\While{$|cand| > 0$} {
$v$ = $cand$.top(); $cand$.pop(); \tcp{$v$ is the nearest point in $cand$ to $q$}
$o$ = $result_L$.bottom(); \tcp{$o$ is the furthest point in $result_L$ to $q$}
\If {$d(v,q) > d(o,q)$} {
break; \tcp{all points in $result_L$ are evaluted}}%endif
\For {each neighbor $e$ of $v$ in $G$}{
\If {$e \notin visited$}{
$cand$.push($e$)\;
$result_L$.push($e$)\;
}%endif
}%endfor
$visited$.push($v$); \tcp{add $v$ to the visited set}
}%endwhile
\Return the top $k$ points in $result_{L}$
\end{algorithm}

Please note that $L$ is often greater than $k$ to achieve better answer quality. For ease of presentation, we assume $k=L$ throughout this paper unless stated otherwise.
\subsection{Review of Graph Search Models and Their Limitations}
\label{section:gap}
While empirical studies demonstrate that the graph-based ANN search algorithms are very competitive, it is widely recognized that the graph-based methods are mostly based on heuristics and not well understood quanlitatively~\cite{Baranchuk19:graphs-construction-by-reinforcement-learning, FuXWC19:nsg}. As an exception, two recent papers take the first step to analyze the asymptotic performance of graph-based methods for datasets uniformly distributed on a $d$-dimensional Euclidean sphere~\cite{Laarhoven18:graph-time-space-tradeoff, Prokhorenkova19:graph-practice-to-theory}. The worst-case analysis shows that the asymptotic behavior of a greedy graph-based search only matches the optimal hashing-based algorithm~\cite{AndoniLRW17:optimal-hash-based-tradeoff}.

It was experimentally observed that the graph-based methods are orders of magnitude faster than the hashing-based algorithms~\cite{AumullerBF20:ann-benchmark}. Thus, though interesting from a pure theoretical perspective, their theory fails to explain the salient practical performance of the graph-based algorithms. Next, we will review several graph/network models that inspire practical graph-based algorithms, and then point out their limitations.

\textbf{Monotonic Search Network Model.} The monotonic search networks (MSNET) are defined as a family of graphs such that, for any two vertices in the graph, there exists at least one monotonic search path between them~\cite{DearholtGK88:msnet}. If the query point happens to be equal
to a point of S, then a simple greedy search will succeed in locating the query point along a path of monotonically decreasing distance to the query point. The original MSNET is not practical, even for datasets of moderate size, because of its $O(n^3)$ indexing complexity and unbounded average out-degree~\cite{DearholtGK88:msnet}. A recent proposal, the \emph{monotonic relative neighborhood graph}, reduces the graph construction time to $O(n^2\log n)$. This, however, still does not make MSNETs applicable in practice.

\textbf{Delaunay Graph Model.} Given a set of elements in a Euclidean space, the Voronoi diagram associates a Voronoi region with each element, which gives rise to a notion of neighborhood. The significance of this neighborhood is that if a query is closer to a database element than all its neighbors, then we have found the nearest element in the whole database~\cite{SebastianK02:metric-based-nn-retrieval, MorozovB18:similarity-graph-for-maximum-inner-product-search}. Delaunay graph is the dual of the Voronoi diagram, where each element $p$ is connected to all elements that share a Voronoi edge with $p$. Using the Delaunay graph, Algorithm~\ref{alg:Graph-based-KNN-Search} is guaranteed to find the NN of $q \notin V$. Unfortunately, the worst-case combinational complexity of the Voronoi diagram in dimension $d$ grows as $\Theta(n^{d/2})$~\cite{PreparataS85:computational-geometry-an-intro}. In addition, the Delaunay graph quickly reduces to the complete graph as $d$ grows, making it infeasible for NN search in high dimension spaces~\cite{Navarro99:nn-searching-by-spatial-approximation}.

\textbf{Navigable Small World Model.} Networks with logarithmic or polylogarithmic complexity of the greedy graph search are known as the navigable small world networks~\cite{Kleinberg00:navigation-in-a-small-world, Kleinberg00:swn-an-algorithmic-perspetive, MalkovPLK14:navigable-small-world-network}. Inspired by this model, Malkov et al. proposed the navigable small world graph (NSW) and hierarchical NSW by introducing ``long" links during the approximate $K$NN graph construction, expecting that the greedy routing achieves polylogrithmic complexity for NN search~\cite{MalkovY20:hierarchical-navigable-SMG}. They demonstrate that the number of hops during the graph routing is polylogrithmic with respect to the network size on a collection of real-life datasets experimentally. However, unlike the ideal navigable small world model, no rigorous theoretical analysis is provided for NSW and HNSW.

We argue that these conceptual models are inadequate in explaining why in most cases the search procedure quickly converges to the nearest neighbor since:
\begin{itemize}

\item For ideal models, the MSNET alone gives no hint how the graph-based methods generalize to out-of-sample queries, i.e., queries that are not in $V$. Delaunay graph supports out-of-sample queries, but do not guarantees NN could be found for query $q$ $\in V$. For example, suppose Algorithm~\ref{alg:Graph-based-KNN-Search} can reach $q$ by traversing one monotonic search path $sv_1v_2 \cdots v_iq$ from $s$ to $q$, we actually have no idea whether $v_i$ is the NN of $q$ at all because there may be multiple monotonic search pathes and NN of $q$ may lie in some other pathes. Navigable small world model only gives intuitive explanations on the existence of short search pathes but have no quantitative justifications that why the NN of $q$ can be found.

  \item More importantly, the vast majority of graph-based algorithms uses approximate $K$NN graph or its variants, instead of the aforementioned conceptual models, as the index structure. By limiting the maximum out-degree, approximate $K$NN graphs are far more sparse than MSNETs and Delaunay graphs, which makes it is fully devoid of the nice the theoretical properties, i.e., the existence of the monotonic search path or Voronoi neighborhood.

\end{itemize}

To sum up, the existing models fail to illuminate the intuitive appeal of the graph-based methods. We view this as a significant gap between the theory and practice of the graph-based search paradigm. In this paper, we try to explain more quantitatively, from a different perspective, why the approximate $K$NN graph-based methods work so well in practice.

\section{Clustering Coefficient of $K$NN Graph and Its Impact on Search Performance}
\label{section:clustering-coefficient}
In graph theory, the \emph{clustering coefficient} is a measure of the degree to which nodes in a graph tend to cluster together, which has been used successfully in its own right across a wide range of applications in complex networks. To name a few, Burt uses local clustering as a probe for the existence of so-called ``structural holes" in a network, and Dorogovtsev et al. found that $C_i$ falls off with $k_i$ approximately as $k_i^{-1}$ for certain models of scale-free networks~\cite{Newman00:networks-an-intro}, where $k_i$ is the degree of $v_i$.

There are different ways to define clustering coefficient. In this paper, we adopt the commonly used definition given by Watts and Strogatz~\cite{Watts98:dynamics-of-small-world-network}. The \emph{local clustering coefficient} $C_i$ of a vertex $v_i$ is defined as
\\

\begin{equation}
\label{equation:local-clustering-coefficient}
C_i  = \frac{\text{number of pairs of neighbors of } v_i \text{ that are connected}}{\text{number of pairs of neighbors of } v_i}
\end{equation}
\\

To calculate $C_i$, we go through all distinct pairs of vertices that are neighbors of $v_i$ in the network, count the number of such pairs that are connected to each other, and divide by the total number of pairs $\frac{k_i(k_i-1)}{2}$. {Figure~\ref{Fig:example-of-clustering-coefficient} illustrates the definition of the local clustering coefficient $C_i$. The degree of $v_i$ is 4 and there exists two edges between its neighbors. Hence, by definition $C_i = 2 \div \frac{ 4 \times (4-1)}{2} = \frac{1}{3}$.}

\begin{figure}[http]
\center
 \includegraphics[width= 0.4\columnwidth]{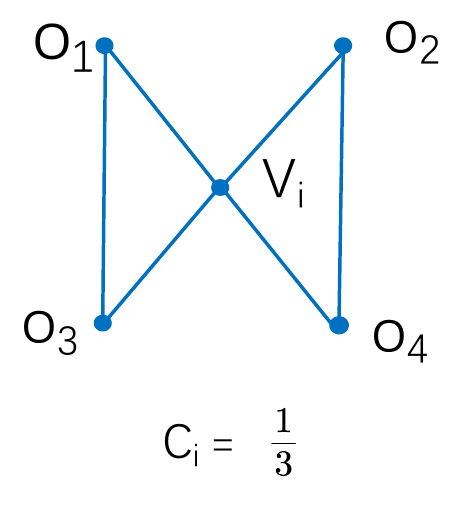}
\caption{\textbf{An illustrative example of local clustering coefficient}}
\label{Fig:example-of-clustering-coefficient}
\end{figure}

The clustering coefficient for the whole network is the average
\\

\begin{equation}
\label{equation:clustering-coefficient}
C  = \frac{1}{n} \sum \limits_{i=1}^{n} C_i
\end{equation}
\\

The local clustering coefficient $C_i$ of a vertex $v_i$ describes the likelihood that the neighbours of $v_i$ are also connected, i.e., the probability that two randomly selected neighbors of $v_i$ are
neighbors with each other.  Roughly speaking, it tells how well the neighborhood of the node is connected. If the neighborhood is fully connected, the local clustering coefficient is 1 and a value close to 0 means that there are hardly any connections in the neighborhood. If most of the nodes in the network have high clustering coefficient, then the network will probably have many edges that connect nodes to each other.

Clustering coefficient of $K$NN graphs depends on $K$ and the intrinsic feature of datasets. Table~\ref{table:clustering-coefficient-for-different-k} lists the clustering coefficients for various $K$ under three typical datasets
%\footnote{100 random points are used to calculate the clustering coefficients}.
As one can see, the larger $K$ is, the greater clustering coefficient will be. Moreover, the relative order of the clustering coefficient is stable independent of $K$. In the sequel, we will use the clustering coefficient in the case of $K=50$ as the default since $K$ cannot be too large due to the index space constraint.

\begin{table*}[htbp]
\large
\caption{Clustering coefficient for different $K$ under three typical datasets}
\centering
\begin{tabular}{c|c|c|c|c|c}
\hline
 Dataset &  $K=20$ & $K=50$ & $K=100$ & $K=150$ & $K=200$\\
\hline
\hline
Sift & 0.1159 & 0.1249 & 0.1371 & 0.1419 & 0.1468\\

Glove & 0.0881 & 0.1029 & 0.1289 & 0.1358 & 0.1427\\

Random & 0.00047 & 0.00074 & 0.00092 & 0.00114 & 0.00139\\

\hline
\end{tabular}
\label{table:clustering-coefficient-for-different-k}
\end{table*}

Our key observation is that the \emph{clustering coefficient} of $K$NN graphs is a informative measure for the efficiency of the graph-based ANN search methods. In this paper, a $K$NN graph is defined as a graph such that for each vertices $v$, there exists bi-directional edges with its $K$ most nearest neighbors. This model is reasonable because practical graph-based algorithms such as HNSW and NSG always add bi-directional links between a point and its $K$NN as much as possible~\cite{MalkovY20:hierarchical-navigable-SMG, FuXWC19:nsg}. Table~\ref{table:clustering-coefficient-vs-efficiency} lists the statistics of the datasets, the clustering coefficients of $K$NN graph in increasing order, the recalls of the top-$50$ query and the average number of hops in the graph for a collection of datasets under HNSW and NSG, the two state-of-the-art graph-based algorithms\footnote{statistics are collected over 1000 random queries}. For both methods, the maximum out-degrees (MOD) is 70 and the parameter $L$, which control the number of hops in the graph, is set to 50. Interesting observations can be made as follows:

\begin{itemize}

\item NSG consistently outperforms HNSW in recall with slightly greater average number of hops, which approximately translates to the number of distance evaluation since the MODs are identical for both algorithms. This observation agrees with the results reported in~\cite{FuXWC19:nsg}.

\item A more interesting observation is that, with around the same number of average hops in the graph, clustering coefficient and recall are strongly correlated. Particularly, The Pearson correlation coefficients between the clustering coefficient and recall for NSG and HNSW are 0.794 and 0.762, respectively. Besides, independent of the data cardinality and dimensionality, high clustering coefficient (greater than 0.12) often leads to high recall whereas low clustering coefficient (lower than 1.0) results in low recall. As an extreme example, the clustering coefficient of Random dataset is only 0.00074 and thus makes graph-based algorithms become very inefficient. One reason that the recall of NSG is greater than that of HNSW is that the quality of neighbors of NSG is better than HNSW, that is, NSG is much closer to an exact $K$NN graph than HNSW. Please note that the datasets are comprehensive enough in terms of size, dimensionality and data types (images, text, audio and synthetic). Detailed description of these datasets can be found in~\cite{LiZSWZL16:DPG}\footnote{https://github.com/DBWangGroupUNSW/}.
\end{itemize}

\begin{table*}[!htbp]
\large
\caption{Clustering coefficient vs. Efficiency}
\centering
\begin{tabular}{c c c c | c c | c c}
\hline
\multirow{2}{*}{Dataset} & \multirow{2}{*}{Size} & \multirow{2}{*}{Dim} & \multirow{2}{*}{Clustering coefficient} & \multicolumn{2}{c|}{HNSW} & \multicolumn{2}{c}{NSG}  \\
\cline{5-8}
     & & & & Recall & \# of Hops & Recall & \# of Hops \\
\hline
\hline
Random & 1,000,000 & 128 & 0.00074 & 0.0049 & 61.3\ & 0.02 & 64.8\ \\
\hline
\hline
Gist & 1,000,000 & 960 & 0.080 & 0.5984 &  54.9\ & 0.7688 & 54.7\ \\

NUSWIDE & 268,643 & 500 & 0.096 & 0.4343 &  58.0 \ & 0.5430 & 59.8\ \\

GLOVE & 1,192,514 & 100 & 0.103 & 0.4903 &  60.3\ & 0.694 & 56.1\ \\

ImageNet & 2,340,373 & 150 & 0.114 & 0.6643 &  53.3\ & 0.8608 & 55.5\ \\

\hline
\hline

Sift & 1,000,000 & 128 & 0.125 & 0.8667 &  52.0\ & 0.9453 & 54.2\ \\

Sun & 79,106 & 512 & 0.140 & 0.8941 &  51.1\ & 0.9562 & 52.0\ \\

Cifar & 50,000 & 512 & 0.141 & 0.9196 &  51.0 \ & 0.9685 & 51.5\ \\

Deep & 1,000,000 & 256 & 0.144 & 0.8205 &  52.7\ & 0.9078 & 54.7\ \\

MillionSong & 992,272 & 420 & 0.163 & 0.5984 &  51.4\ & 0.9608 & 55.1\ \\

Ukbench & 1,097,907 & 128 & 0.189 & 0.8893 &  51.7\ & 0.9545 & 54.5\ \\

Enron & 94,987 & 1369 & 0.209 & 0.7599 &  52.3\ & 0.9421 & 53.3\ \\

Trevi & 99,900 & 4096 & 0.215 & 0.8845 &  51.1\ & 0.9498 & 52.8\ \\

AUDIO & 53,387 & 192 & 0.253 & 0.9553 &  51.0\ & 0.9815 & 52.5\ \\

MINIST & 69,000 & 784 & 0.286 & 0.9728 &  51.7\ & 0.9878 & 53.2\ \\

Notre & 332,668 & 128 & 0.287 & 0.9248 &  52.4\ & 0.9674 & 53.8\ \\
\hline
\end{tabular}
\label{table:clustering-coefficient-vs-efficiency}
\end{table*}

These observations suggest that the clustering coefficient are a promising measure for the efficiency of graph-based algorithms. Intuitively, higher the clustering coefficient of $K$NN graph is\footnote{$K$ should be as small as possible to reduce the memory footprint and query efficiency}, the better the graph is connected, which means that graph connectivity has significant impact on the result quality of ANN search. To have an in-depth understanding of how connectivity affects the search performance, we scrutinized the graph traversal steps of a sample of queries and found that the local connectivity, instead of the global one, is the determining factor. To formally characterize the local connectivity, we propose the notion of \emph{maximum strongly connected neighborhood} as follows.

\begin{definition}
\label{definition:strongly-connected-component}
A directed graph is strongly connected if there is a path between all pairs of vertices. A strongly connected component (SCC) of a directed graph is a strongly connected subgraph in this graph.
\end{definition}

\begin{definition}
\label{definition:neighborhood-of-a-node}
The $k$-neighborhood of a point $v$, denoted by $\mathcal{N}_k (v)$, is the set of $k$ nearest elements of $v$, i.e., $o_1  \cdots o_k \in V$.
\end{definition}

Please note that the only requirement is the $k$ nearest neighbors of $v$ belongs to $V$ and $v$ may be some point in or not in $V$. This definition makes our analysis supports out-of-sample queries.

\begin{definition}
\label{definition:neighborhood-subgraph}
A subgraph of $G$ is the $k$-neighborhood subgraph associated with a vertex $v$, denoted by $\mathcal{G}_k (v)$, if $V(\mathcal{G}_k (v)) = \mathcal{N}_k (v)$ and $E(\mathcal{G}_k (v)) \subseteq E(G)$.
\end{definition}

\begin{definition}
\label{definition:strongly-connected-component}
The maximum strongly connected neighborhood of $\mathcal{G}_k (v)$, denoted by $\mathcal{C}_k (v)$, is an SCC of $\mathcal{G}_k (v)$ such that $|\mathcal{C}_k (v)| \ge |\mathcal{C}_i|$ for all $i$, where $\mathcal{C}_i$ are the SCCs of $\mathcal{G}_k (v)$.
\end{definition}

Please note that $K$ and $k$ owns totally different meaning - $K$ is the link number per node of $K$NN graph and is determined in graph construction and $k$ is the search parameter and thus may vary according to the users' requirement.

{Figure~\ref{Fig:example-of-strongly-connected-neighborhood} illustrates these definitions using a simple example. $o_1$ to $o_5$ is the top-5 NN of query $q$ and there are three undirected edges (equivalent to six directed edges) in $q$'s 5-neighborhood subgraph $\mathcal{G}_5 (q)$. Three SCCs exists in $\mathcal{G}_5 (q)$ and the maximum SCC $\mathcal{C}_5 (q)$ is composed of point $o_1$, $o_2$ and $o_3$.}

\begin{figure}[http]
\center
 \includegraphics[width= 0.5\columnwidth]{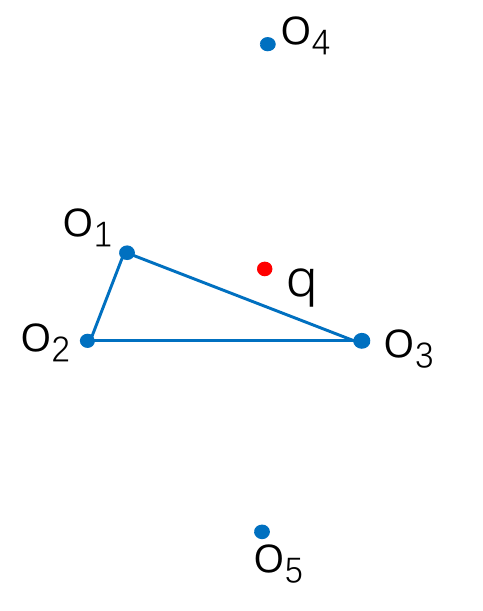}
\caption{\textbf{An illustrative example of maximum strongly connected neighborhood}}
\label{Fig:example-of-strongly-connected-neighborhood}
\end{figure}

To show the impact of $\mathcal{C}_k (v)$ on the algorithm performance. Table \ref{table:SCC-of-a-random-query-for-sift} lists the 3 SCCs of largest sizes for 100 random $k$NN queries in the case of $k=50$ for three typical datasets. As we can see, the ratios of the size of $\mathcal{C}_k (v)$ (SCC1) to $k$ are very close the recall listed in Table~\ref{table:clustering-coefficient-vs-efficiency} for three datasets, respectively. Other state-of-the-art algorithms, such as HNSW, exhibits similar trends.

\begin{table*}[!htbp]
\large
\caption{SCCs of 100 random queries for three typical datasets under NSG}
\centering
\centering
\begin{tabular}{c || c c | c c | c c }
\hline
\multirow{2}{*}{SCC-id} & \multicolumn{2}{c|}{Sift} & \multicolumn{2}{c|}{Glove} & \multicolumn{2}{c}{Random} \\
\cline{2-7}
     & size & ratio & size & ratio & size & ratio\\
\hline
SCC1 & 47.8 & 95.6\% & 33.8 & 67.6\% & 2.4 & 4.8\%\\
SCC2 & 0.8 & 1.6\% & 1.8 & 3.6\% & 1.4 & 2.8\%\\
SCC3 & 0.2 & 0.4\% & 1.8 & 3.6\% & 1.2 & 2.4\% \\
\hline
\end{tabular}
\label{table:SCC-of-a-random-query-for-sift}
\end{table*}

To eliminate the bias caused by specific graph construction algorithms, we studied the exact $K$NN graph and found similar results. Table~\ref{table:ideal-directional-maximum-strongly-connected-neighborhood} lists, for 100 random top-50 NN queries, the average sizes of the top-3 SCCs over Sift, Glove and Random. In this experiment, we only put a directed link from a point to its $K$NN and no link is added manually in the reverse direction, i.e., the $K$NN graph is directed. $K$ is set to 50. From Table~\ref{table:ideal-directional-maximum-strongly-connected-neighborhood} we can see that, independent of specific graph-based algorithm, clustering coefficient also has a significant impact on the size of $\mathcal{C}_k (v)$.

\begin{table*}[!htbp]
\large
\caption{SCCs of 100 random queries for exact directed $K$NN graph}
\centering
\centering
\begin{tabular}{c || c c | c c | c c }
\hline
\multirow{2}{*}{SCC-id} & \multicolumn{2}{c|}{Sift} & \multicolumn{2}{c|}{Glove} & \multicolumn{2}{c}{Random} \\
\cline{2-7}
     & size & ratio & size & ratio & size & ratio\\
\hline
SCC1 & 36.8 & 73.6\% & 30.2 & 60.4\% & 1.2 & 2.4\%\\
SCC2 & 3.6 & 7.2\% & 2.4 & 4.8\% & 1 & 2\%\\
SCC3 & 1.6 & 3.2\% & 1.6 & 3.2\% & 1 & 2\% \\
\hline
\end{tabular}
\label{table:ideal-directional-maximum-strongly-connected-neighborhood}
\end{table*}

We also examined undirected $K$NN graph, where bi-directional link is added manually between a point and its $K$NN. The trend listed in Table~\ref{table:ideal-bi-directional-maximum-strongly-connected-neighborhood} is very similar to that of Table~\ref{table:ideal-directional-maximum-strongly-connected-neighborhood} except that the sizes of $\mathcal{C}_k (v)$ are larger. This is because more links are added in the graph. Actually, practical graph-based algorithms lie somewhere between the undirected and directed $K$NN graph since they always try to add bi-directional links as long as the memory budget is enough. Please note exact $K$NN graphs are not practical because the unaffordable construction time and unlimimted maximum out-degree, which translates to too much memory cost.

\begin{table*}[!htbp]
\large
\caption{SCCs of 100 random queries for exact undirected $K$NN graph}
\centering
\centering
\begin{tabular}{c || c c | c c | c c }
\hline
\multirow{2}{*}{SCC-id} & \multicolumn{2}{c|}{Sift} & \multicolumn{2}{c|}{Glove} & \multicolumn{2}{c}{Random} \\
\cline{2-7}
     & size & ratio & size & ratio & size & ratio\\
\hline
SCC1 & 48.8 & 97.6\% & 46.2 & 92.4\% & 6 & 12\%\\
SCC2 & 0.2 & 0.4\% & 1 & 2\% & 2.2 & 4.4\%\\
SCC3 & 0 & 0\% & 0.2 & 0.4\% & 1.4 & 2.8\% \\
\hline
\end{tabular}
\label{table:ideal-bi-directional-maximum-strongly-connected-neighborhood}
\end{table*}

In a nutshell, all these experiments demonstrate that clustering coefficient of $K$NN graph is an informative measure for the size of the maximum strongly connected neighborhood and the performance of graph-based algorithms over a specific dataset. Next, we will analyze how $\mathcal{C}_k (q)$ affects the recall for a given query $q$. Particularly, we will show that Algorithm~\ref{alg:Graph-based-KNN-Search}, the striking algorithmic component for graph search, can effectively reach $\mathcal{C}_k (q)$ and identify all $k$NN $\in \mathcal{C}_k (q)$. This explains why greater clustering coefficient and larger size of $\mathcal{C}_k (q)$ lead to better performance.

\section{Two Phase $k$NN Search in Graphs}
\label{section:the-analysis-of-search-algorithm}

\begin{table*}[htbp]
\large
\caption{Statistics of two-phase ANN search for Sift}
\centering
\begin{tabular}{c|c|c|c|c|c}
\hline
Query ID & \# of Hops in $P_1$ & \# of Hops in $P_2$ & $|\mathcal{C}_k (q)|$ & Fraction of $|\mathcal{C}_k (q)|$ visited in $P_2$ & $|\overline{\mathcal{C}_k (q)}|$\\
\hline
\hline
1 & 4 & 50 & 38 & 100\% & 4\\

2 & 4 & 50 & 43 & 100\% & 1\\

3 & 5 & 50 & 48 & 100\% & 0\\

4 & 4 & 50 & 50 & 100\% & 0\\

5 & 5 & 50 & 45 & 100\% & 0\\
\hline
\end{tabular}
\label{table:an-illustration-of-two-phase-search-for-sift}
\end{table*}

\begin{table*}[htbp]
\large
\caption{Statistics of two-phase ANN search for GLOVE}
\centering
\begin{tabular}{c|c|c|c|c|c}
\hline
Query ID & \# of Hops in $P_1$ & \# of Hops in $P_2$ & $|\mathcal{C}_k (q)|$ & Fraction of $|\mathcal{C}_k (q)|$ visited in $P_2$ & $|\overline{\mathcal{C}_k (q)}|$\\
\hline
\hline
1 & 3 & 52 & 27 & 100\% & 7\\

2 & 1 & 59 & 28 & 100\% & 0\\

3 & 11 & 39 & 20 & 100\% & 8\\

4 & 2 & 49 & 38 & 100\% & 4\\

5 & 2 & 48 & 24 & 100\% & 6\\
\hline
\end{tabular}
\label{table:an-illustration-of-two-phase-search-for-glove}
\end{table*}

\begin{table*}[htbp]
\large
\caption{Statistics of two-phase ANN search for Random}
\centering
\begin{tabular}{c|c|c|c|c|c}
\hline
Query ID & \# of Hops in $P_1$ & \# of Hops in $P_2$ & $|\mathcal{C}_k (q)|$ & Fraction of $|\mathcal{C}_k (q)|$ visited in $P_2$ & $|\overline{\mathcal{C}_k (q)}|$\\
\hline
\hline
1 & 79 & 7 & 1 & 100\% & 0\\

2 & 83 & 0 & 1 & 0\% & 0\\

3 & 50 & 47 & 1 & 100\% & 0\\

4 & 46 & 25 & 2 & 100\% & 0\\

5 & 56 & 0 & 1 & 0\% & 0\\
\hline
\end{tabular}
\label{table:an-illustration-of-two-phase-search-for-random}
\end{table*}

The common wisdom about Algorithm~\ref{alg:Graph-based-KNN-Search} is as follows. Starting from the entry vertex $s$, which is chosen by random or using some auxiliary method, Algorithm~\ref{alg:Graph-based-KNN-Search} finds a directed path from $s$ to the query $q$, hoping that NN of $q$ are identified through the walk. Since only local information, i.e., adjacent vertices of the visited vertices, is used, this class of algorithms are termed as the \emph{decentralized} algorithm~\cite{Kleinberg00:navigation-in-a-small-world}. Particularly, for ANN search, Algorithm~\ref{alg:Graph-based-KNN-Search} first follows the out-edges of $s$ to get its immediate neighbors, and then examines the distances from these neighbors to $q$. The one with the minimum distance to $q$ is selected as the next base vertex for iteration. The same procedure is repeated at each step of the traversal until Algorithm~\ref{alg:Graph-based-KNN-Search} reaches a local optima, namely, the immediate neighbors of the base vertex does not contain a vertex that is closer to $q$ than the base vertex itself. Backtracking is used to jump out of the local optima and increase the odd to find the true NN. Recall that we name this search paradigm as multiple path search model.

Different from the traditional point of view, we observe that Algorithm~\ref{alg:Graph-based-KNN-Search} actually is composed of two phases. In the first phase, the algorithm starts with an initial point, walks the graph and encounters a point within $\mathcal{C}_k (q)$. In the second phase, the algorithm traverse $\mathcal{C}_k (q)$ and a small number of points not in $\mathcal{C}_k (q)$. {Figure~\ref{Fig:two-phases-of-nn-search}} depicted the two phase search procedure.  Theorem~\ref{theorem:traverse-the-strongly-connected-component} proves that Algorithm~\ref{alg:Graph-based-KNN-Search} is guaranteed to find all points in $\mathcal{C}_k (q)$ under a mild condition.

\begin{figure}[http]
\center
 \includegraphics[width= 0.8\columnwidth]{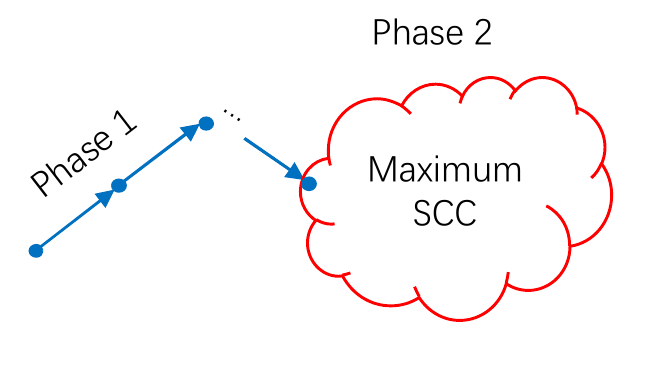}
\caption{\textbf{Two-phases ANN graph search}}
\label{Fig:two-phases-of-nn-search}
\end{figure}

\begin{theorem}
\label{theorem:traverse-the-strongly-connected-component}
Algorithm~\ref{alg:Graph-based-KNN-Search} is guaranteed to visit all points in $\mathcal{C}_k (q)$ starting with any point in $\mathcal{C}_k (q)$.
\end{theorem}

\begin{proof}
We know that any directed graph is said to be a strongly connected component iff all the vertices of the graph are a part of some cycle. Please note that the cycle may not necessarily be a Hamilton cycle. Suppose all vertices not in $\mathcal{C}_k (q)$ but adjacent with vertices in $\mathcal{C}_k (q)$ are further to $q$ than all vertices in $\mathcal{C}_k (q)$. Without loss of generality, suppose the first vertex visited is $v_1$, then Algorithm~\ref{alg:Graph-based-KNN-Search} will visit all vertices following the cycle and push every vertices into $cand$ and $result_{L}$. Since all vertices in $\mathcal{C}_k (q)$ are closer to $q$ than other vertices, the distance of the bottom element of $result_{L}$ will always greater than that of the top element in $cand$ until all vertices in $\mathcal{C}_k (q)$ are visited\footnote{all elements in $result_{L}$ are initialized as infinity at the beginning}. Please note that in each loop $cand$ pop up the element that have been pushed into $result_{L}$, which guarantees that Algorithm~\ref{alg:Graph-based-KNN-Search} always terminates.
\end{proof}

Theorem~\ref{theorem:traverse-the-strongly-connected-component} suggests a different perspective in understanding the graph-based methods. Rather than searching a single path (without backtracing) or multiple paths (with backtracking) in the graph, the search algorithm actually traverses a strongly connected neighborhood around the query. In other words, high quality $\mathcal{C}_k (q)$, together with Algorithm~\ref{alg:Graph-based-KNN-Search}, offers the salient performance. The analysis in Section~\ref{section:clustering-coefficient} reveals the quality of $\mathcal{C}_k (q)$ are data dependent and closely related to the clustering coefficient of $K$NN graphs. Therefore, there exists significant performance disparity for different datasets and we could use the clustering coefficient of $K$NN graph as a meaningful measure for the efficiency of the graph-based methods.

It is possible that there exists a few vertices adjacent with vertices in $\mathcal{C}_k (q)$, which is not in $\mathcal{C}_k (q)$ but  closer to $q$ than some vertices in $\mathcal{C}_k (q)$. In this case, the algorithm may also visit such vertices, and the answer quality will be higher than just traversing $\mathcal{C}_k (q)$ since more closer vertices outside $\mathcal{C}_k (q)$ are visited.

The probability of the search algorithm getting into $\mathcal{C}_k (q)$ is exponentially increasing with $L^\prime$, the number of being trapped into a local optima and getting back to a distant point before entering $\mathcal{C}_k (q)$, which is expressed as follows. $p_i$ is the probability of getting into $\mathcal{C}_k (q)$ along a single path.

\begin{equation}
\label{equation:prob-of-entering-the-second-phase}
P  = 1- \prod \limits_{i=1}^{L^\prime} (1-p_i)
\end{equation}

The rigorous calculation of $P$ is infeasible. Empirically, for datasets with relatively large clustering coefficients we observed that (1) Algorithm~\ref{alg:Graph-based-KNN-Search} can quickly reach $\mathcal{C}_k (q)$, and (2) the path length of the first phase is far shorter than that of the second phase. Table~\ref{table:an-illustration-of-two-phase-search-for-sift}, Table~\ref{table:an-illustration-of-two-phase-search-for-glove} and Table~\ref{table:an-illustration-of-two-phase-search-for-random} list the numbers of hops in Phase 1 and Phase 2, the size of $\mathcal{C}_k (q)$, the fraction of points in $\mathcal{C}_k (q)$ that are visited in Phase 2 and the number of true top-$k$ NN not in $\mathcal{C}_k (q)$ that are found during Phase 2 (denoted by $|\overline{\mathcal{C}_k (q)}|$) for Sift, Glove and Random, respectively. NSG is used and the statistics of five random query are reported. Please note that HNSW and exact $K$NN graph exhibit the similar trends thus we do not report results for them. Several interesting observations are made:

\begin{itemize}

\item As we can see, independent of datasets, the two-phase search model is applicable to all queries listed. As proved in Theorem~\ref{theorem:traverse-the-strongly-connected-component}, once the search algorithm enters Phase 2, all points in $\mathcal{C}_k (q)$ will be visited, which demonstrates the importance of the quality of $\mathcal{C}_k (q)$.

\item Besides the true top-$k$ NN in $\mathcal{C}_k (q)$, other true top-$k$ NN also probably be visited during Phase 2, especially for Glove where the size of $\mathcal{C}_k (q)$ is relatively small. This is mainly caused by the search algorithm jumps into a smaller SCC or visits $k$NN that only own a single directed edge with the maximum SCC.

\item For Sift and Glove, where the size of $\mathcal{C}_k (q)$ far greater than that of Random, the second phase dominates the search cost and the algorithm jumps into $\mathcal{C}_k (q)$ in a very small number of steps. In contrast, it is very hard for the algorithm to find a true top-$k$ NN for Random since the size of $\mathcal{C}_k (q)$ is too small (in most case is equal to 1). For example, $q_2$ and $q_5$ do not enter Phase 2 and didn't find any true NN. As a result, the recall of Random is very low.
\end{itemize}

\begin{figure}[http]
\center
 \includegraphics[width= 1\columnwidth]{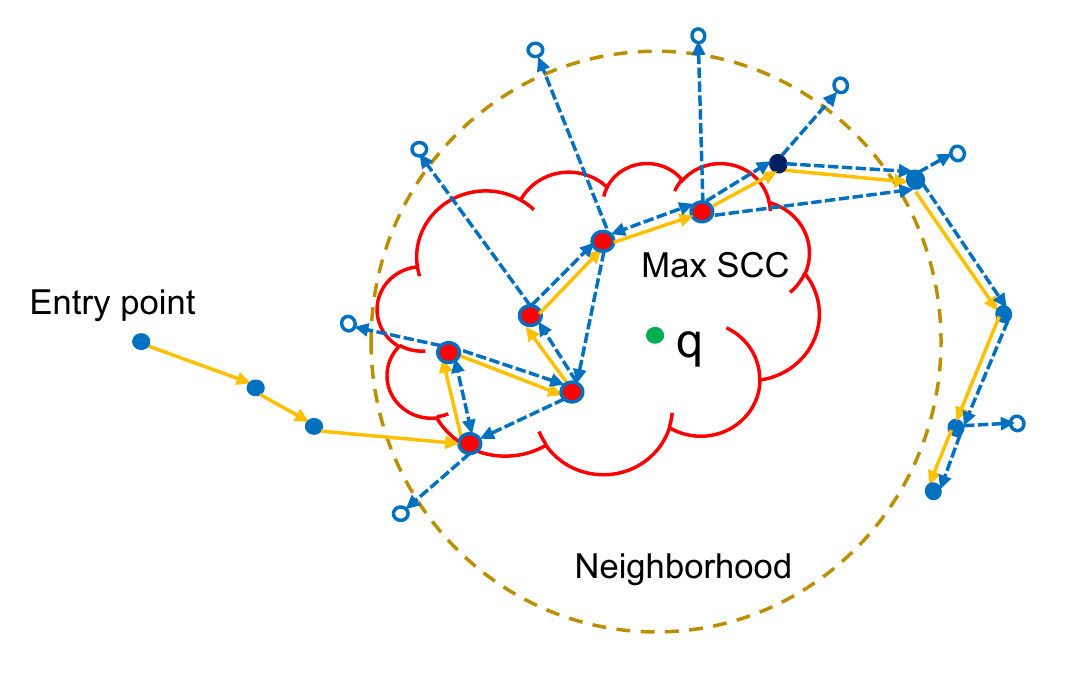}
\caption{\textbf{An example of two-phase ANN graph search (best viewed in color)}}
\label{Fig:example-two-phases-of-nn-search}
\end{figure}

To train reader's intuition, Figure~\ref{Fig:example-two-phases-of-nn-search} illustrates the search procedure of a top-10 query for Sift dataset with NSG. Green point is the query and red points denote the true top-NN in the maximum SCC, which are strongly connected. Dashed lines in blue with single or double arrows represents the the directed or bi-directional edges between points. The solid arrowed lines in yellow depict the search path during $k$NN search. As we can see, starting with the entry point, the algorithm jumps into the maximum SCC in three steps. After traversing the maximum SCC which consists of six true NNs, it continues the search by visiting one true NN (in black) and four other points before the termination condition is met. Since $k$ is small in this example, the size of the maximum SCC is not that large. This can be informally explained by the fact that the connectivity increases as the number vertices become large under the same edge connection probability using random network theory~\cite{Newman00:networks-an-intro}.

\textbf{The case of small $k$:} Users may be only interested in a small number of nearest neighbors of $q$, say $k$ ranging from 1 to 10. In this case, the size and quality of $\mathcal{C}_k (q)$ is not that good to achieve high recall. To get precise results, $L$ is often set greater than $k$, say 50-200. The net effect is that the search algorithm actually visit $\mathcal{C}_L (q)$, which consists the most points of top-$L$ NN if the clustering coefficient is large enough, and then Algorithm 2 identify the best $k$ results and output them.

\section{Related Work}
\label{section:related}

\subsection{Measrues for Difficulty of Nearest Neighbor Search}
The problem of the difficult (approximate) of NN search in a given dataset has drawn much attention in recent years. Beyer et al. and Francois et al. show that NN search wil be meaningless when the number of dimensions goes to infinity~\cite{BeyerGRS99:when-NN-meaningful, FrancoisWV07:the-concentration-of-NN-search}, respetively. However, they didn't discuss the non-asymptotic analsisi when the number of dimensions is finite. Moreover, the effect of other crucial properties such as the sparsity of data vectors has not been studied. To the best of our knowledege, He et al. proposed the first concrete measure called Relative Contrast (RC) to evaluate the influence of several data characteristics such as dimensionality, sparsity and dataset size simultaneously on the difficulty of NN search~\cite{He2012:on-the-difficulty}. They present a theoretical analysis to prove how RC determines the complexity of Locality Sensitive Hashing, a popular approximate NN search method. Relative Constrast aslo provides an explanation for a family of heristic hashing algorithm with good practical performance based on PCA. However, no evidence is given that RC can be used to explain the success of graph-based NN search method directly.

Identifying the intrinsic dimensionality (ID) of datasets has been studied for decades since its importance in machine learning, databases and data mining. Recently, local ID gains much attention since it is very useful when data is composed of heterogeneous manifolds. In addition to applications in manifold learning, measures of local ID have been used in the context of evaluate the difficulty of NN search~\cite{HouleN15:ranking-based-similarity-search}. Several local intrinsic dimensionality models have been proposed, such as the expansion dimension (ED)~\cite{DBLP:conf/stoc/KargerR02}, the generalized expansion dimension (GED)~\cite{DBLP:conf/icdm/HouleKN12}, the minimum neighbor distance (MiND)~\cite{DBLP:journals/ml/RozzaLCCC12}, local continuous intrinsic dimension (LID)~\cite{DBLP:conf/icdm/Houle13}. While these measures have been shown useful in their own right, non of them is applicable in explaining the salient performance of the graph-based methods.

\subsection{A Brief Review of the Existing ANN Search Methods}

Approximate nearest neighbor search (ANNS) has been a hot topic over decades, it provides fundamental support for many applications of data mining, databases and information retrieval \cite{ArefCFEHIMZ02:video-database, FaginKS03:efficient-similarity-search, KeSH04:efficient-near-duplicate-image-retrieval-lsh}.
There is a large amount of significant literature on algorithms for approximate nearest neighbor search, which are mainly divided into the following categories: tree-structure based approaches, hashing-based approaches, quantization-based approaches, and graph-based approaches.

\subsubsection{tree-structure based approaches}
Hierarchical structures (tree) based methods offer a natural way to continuously partition a dataset into discrete regions at multiple scales, such as KD-tree \cite{Bentley90:k-d-tree}, R-tree \cite{Guttman84:r-tree}, SR-tree \cite{KatayamaS97:sr-tree}. These methods perform very well when the
dimensionality of the data is relatively low. However, it has been proved to be inefficient when the dimensionality of data is high. It has been shown in \cite{WeberSB98:a-quantitative-analysis-for-nn-search} that when the dimensionality exceeds about 10, existing indexing data structures based on space partitioning are slower than the brute-force, linear-scan approach.
Many new hierarchical-structure-based methods \cite{RamS19:revisiting-kd-tree} are presented to address this limitation.

\subsubsection{hashing-based approaches}
Among the approximate NN search algorithms, the Locality Sensitive Hashing is the most widely used one due to its excellent theoretical guarantees and empirical performance. E2LSH, the classical LSH implementations for $\ell_2$ norm, cannot solve $c$-ANN search problem directly. In practice, one has to either assume there exists a ``magical' radius $r$, which can lead arbitrarily bad outputs, or uses multiple hashing tables tailored for different raddi, which may lead to prohibitively large space consumption in indexing. To reduce the storage cost, LSB-Forest~\cite{TaoYSK09:lsb-tree} and C2LSH~\cite{GanFFN12:collision-counting} use the so-called virtual rehashing technique, implicitly or explicitly, to avoid building physical hash tables for each search radius. The index size of LSB-Forest is far greater than that of C2LSH because the former ensures that the worst-case I/O cost is sub-linear to both $n$ and $d$ whereas the latter has no such guarantee - it only bounds the number of candidates by some constant but ignores the overhead in index access.

Based on the idea of query-aware hashing, the two state-of-the-art algorithms QALSH and SRS further improve the efficiency over C2LSH by using different index structures and search methods, respectively.
SRS uses an $m$-dimensional $R$-tree (typically $m \le 10$) to store the
$\langle g(o), oid \rangle$ pair for each point $o$ and transforms the
$c$-ANN search in the $d$-dimensional space into the range query in the
$m$-dimensional projection space. The rationale is that the probability that a point $o$ is the NN of $q$ decreases as $\Delta_m(o)$ increases, where $\Delta_{m}(o) = \|g_m(q)-g_m(o)\|$ During $c$-ANN search, points are
accessed according to the increasing order of their $\Delta_m(o)$.

Motivated by the observation that the optimal
$\ell_p$ metric is application-dependent, LazyLSH~\cite{ZhengGTW16:LazyLSH} is proposed to solve the NN search problem for the fractional distance metrics, i.e., $\ell_p$ metrics ($0 < p < 1$) with a single index. FALCONN is the state-of-the-art LSH scheme for the angular distance, both theoretically and practically~\cite{AndoniILRS15:FALCONN}. Except for E2LSH and FALCONN, the other algorithms are disk-based and thus can handle datasets that do not fit into the memory.

All of the aforementioned LSH algorithms provide probability guarantees on the result quality (recall and/or precision). To achieve better efficiency, many LSH extensions such as Multi-probe LSH~\cite{LvJWCL07:multi-probe-lsh}, SK-LSH~\cite{LiuCHLS14:sk-lsh}, LSH-forest~\cite{BawaCG05:lsh-forest} and Selective hashing~\cite{GaoJOW15:selective-hashing} use heuristics to access more plausible buckets or re-organize datasets, and do not ensure any LSH-like theoretical guarantee.

\subsubsection{quantization-based approaches}
The most common quantization-based methods is product quantization (PQ) \cite{JegouDS11:product-quantization}. It seeks to perform a similar dimension reduction to hashing algorithms, but in a way that better retains information about the
relative distances between points in the original vector space. Formally, a quantizer is a function q mapping a $D$-dimensional vector $x\in \mathbb{R}^{D}$ to a vector $q(x)\in C = \{c_i; i \in \mathcal{I}\}$, where the index set $\mathcal{I}$ is finite: $\mathcal{I}=0 \ldots k-1$. The reproduction values $c_i$ are called centroids. The set $\mathcal{V}_{i}$ of vectors mapped to given index $i$ is referred to as a cell, and defined as
\begin{displaymath}
  \mathcal{V}_{i} \triangleq\left\{x \in \mathbb{R}^{D}: q(x)=c_{i}\right\}
\end{displaymath}
The $k$ cells of a quantizer form a partition of $\mathbb{R}^{D}$. So all the vectors lying in the same cell $\mathcal{V}_{i}$ are reconstructed by the same centroid $c_i$. Due to the huge number of samples required and the complexity of learning the quantizer, PQ uses $m$ distinct quantizers to quantize the subvectors separately. An input vector will be divided into m distinct subvectors $u_j$, $1 \leq j \leq m$. The dimension of each subvector is $D^{*} = D/m$. An input vector x is mapped as follows:
\begin{displaymath}
\footnotesize
  \underbrace{x_{1}, \ldots, x_{D^{*}}}_{u_{1}(x)}, \cdots, \underbrace{x_{D-D^{*}+1}, \ldots, x_{D}}_{u_{m}(x)} \rightarrow q_{1}\left(u_{1}(x)\right), \ldots, q_{m}\left(u_{m}(x)\right)
\end{displaymath}
where $q_j$ is a low-complexity quantizer associated with the $j^{th}$ subvector. And the codebook is defined as the Cartesian product,
\begin{displaymath}
  \mathcal{C}=\mathcal{C}_{1} \times \ldots \times \mathcal{C}_{m}
\end{displaymath}
and a centroid of this set is the concatenation of centroids of the $m$ subquantizers. All subquantizers have the same finite number $k^{*}$ of reproduction values, the total number of centroids is $k=\left(k^{*}\right)^{m}$.

After using PQ, all database vectors will be replaced by reproduction values. In order to speed up the query, PQ proposes a look-up table to directly get the distance between the reproduction values and the query vector. They propose two methods to compute an approximate Euclidean distance between these vectors: the so-called Asymmetric Distance Computation (ADC) and the Symmetric Distance Computation (SDC). See Figure~\ref{Fig:SDC-and-ADC} for an illustration. Let’s take the introduction of ADC as an example.
\begin{figure}[t]
  \centering
  \includegraphics[scale = 0.2]{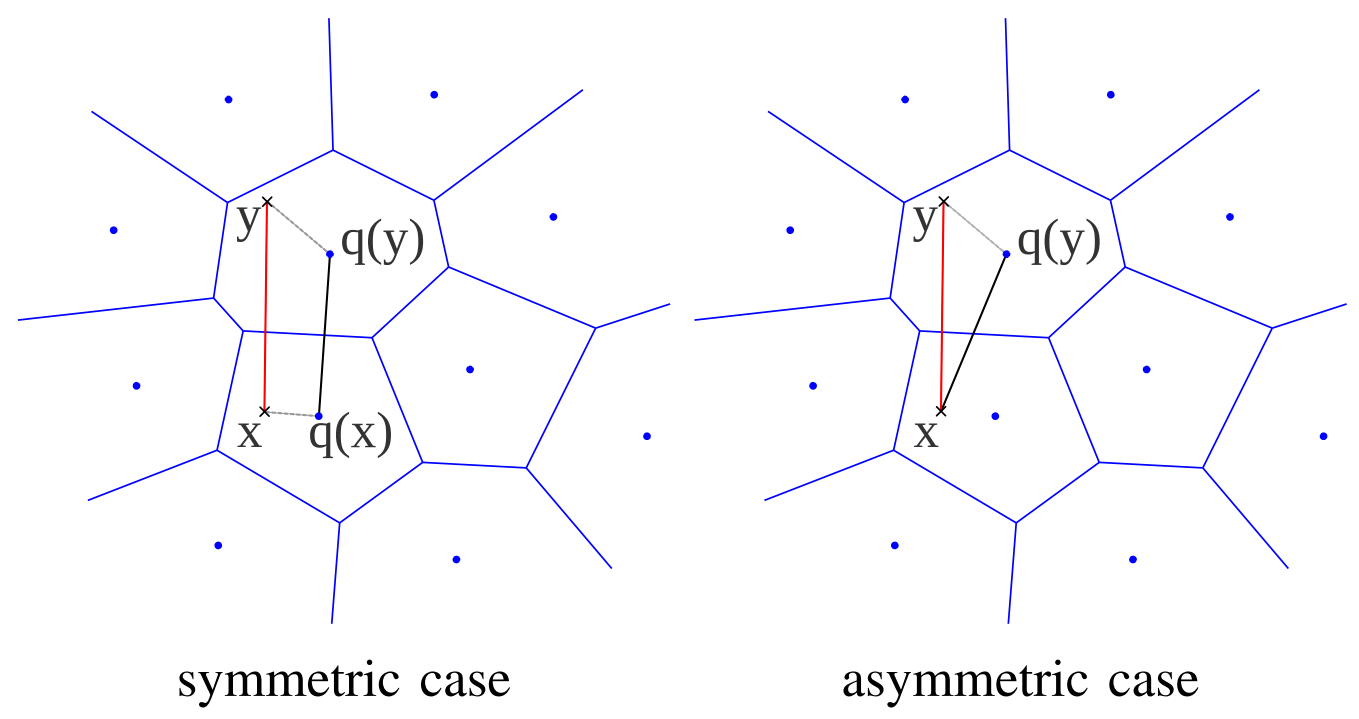}
  \caption{Two methods to compute an approximate Euclidean distance}
  \label{Fig:SDC-and-ADC}
\end{figure}

The database vector $y$ is represented by $q(y)$, but the query $x$ is not encoded. The distance $d(x,y)$ is approximated by the distance $d(x,q(y))$, which is computed using the decomposition
\begin{displaymath}
  d(x, q(y))=\sqrt{\sum_{j} d\left(u_{j}(x), q_{j}\left(u_{j}(y)\right)\right)^{2}},
\end{displaymath}
where the squared distances $d\left(u_{j}(x), c_{j, i}\right)^{2}: j= 1 \ldots m, i=1 \ldots k^{*}$, are computed before the search.
The calculation method of SDC is similar to ADC, but the query vector $x$ is represented by $q(x)$. SDC limits the memory usage associated with the queries and ADC has a lower distance distortion for a similar complexity.

PQ offers three attractive properties: (1) PQ compresses an input vector into a short code (e.g., 64-bits), which enables it to handle typically one
billion data points in memory; (2) the approximate distance between a raw vector and a compressed PQ code is computed efficiently (the so-called asymmetric distance computation (ADC) and the symmetric distance computation (SDC)), which is a good estimation
of the original Euclidean distance; and (3) the data structure and coding algorithm are simple, which allow it to hybridize with other indexing structures. Becasue these methods avoid distance calculations on the
original data vectors, it will cause a loss of certain calculation accuracy. When the recall rate is close to 1.0, the required length of the candidate list is close to the size of the dataset. Many quantization-based methods
try to reduce quantization errors to improve calculation accuracy, such as Optimal Product Quantization (OPQ) \cite{ge2013optimized} and Tree Quantization (TQ) \cite{babenko2015tree}.

\subsubsection{graph-based approaches}
Recently, graph-based methods have drawn considerable attention, such as NSG \cite{FuXWC19:nsg}, HNSW \cite{MalkovY20:hierarchical-navigable-SMG}, Efanna \cite{fu2016efanna}, and FANNG \cite{HarwoodD16:fanng}.
Graph-based methods construct a $k$NN graph offline, which can be regard as a big network graph in high-dimensional space. However, the construction complexity of the exact kNN graph, especially when it comes to large datasets, will increase exponentially. Many researchers turn to building an approximated $k$NN graph, but it is still time consuming. There are two main types of graphs: directed graphs and undirected graphs.

%Given two nodes p, q on the undirected graph G, if p and q have edges on the graph, then p and q are mutually reachable, i.e., p and q are neighbors to each other. But it is possible that p is not a neighbor of q, i.e., this edge is a directed edge. That is because q may have a closer neighbor point than p. Some researchers are committed to building directed graphs, such as NSG \cite{FuXWC19:nsg}.

At online search stage, they all use greedy-search algorithm or its variants. While these method require to find the initial point in advance, and the easiest way is to choose randomly. During the search, it can quickly converge from the initial point to the neighborhood of the query point.
But one problem of this method is that it is easily to converge to local optima and result in a low recall. One way to solve this problem is to provide better initialization candidate set for a query point. Instead of using random selection, choosing to use the Navigating node (the approximate medoid of the dataset) and its neighbors as the candidate.
Another method is to try to make the constructed graph monotonous. The edge selection strategy of MRNG, which was first proposed in paper \cite{FuXWC19:nsg}, can ensure that the graph is a Monotonic Search Network (MSNET). Ideally, the search path will iterate from the starting point until reaching the query point and ending, this means that no backtracking occurs during the search.

Because the construction of graphs greatly affects search performance, many researchers focus on constructing index graphs. The fundamental issue is how to choose the neighbors of nodes on the graph. We will introduce two state-of-the-art graph neighbor selection strategies: Relative Neighborhood Graphs (RNG) \cite{Jaromczyk:rng-and-their-relatives} and Monotonic Relative Neighborhood Graphs (MRNG) \cite{FuXWC19:nsg}.
Formally, given two points p and q in $\mathbb{R}^{D}$ space, $B(p, \delta(p, q))$ represents an open sphere where the center is q, and $\delta(p, q)$ is the radius. The ${lune}_{p q}$ is defined as:
\begin{displaymath}
   {lune}_{p q}=B(p, \delta(p, q)) \cap B(q, \delta(p, q))
\end{displaymath}
FANNG \cite{HarwoodD16:fanng} and HNSW \cite{MalkovY20:hierarchical-navigable-SMG} adopt the RNG's edge selection strategy to construct the index. RNG is an edge selection strategy based on undirected graph, and it selects edges by checking whether there is a point in the intersection of two open spheres. In Figure~\ref{Fig:rng-and-mrng}(a), node $p$ has prepared a set of neighbor candidates for selection. If there is no node in $lune_{pr}$, then $p$ and $r$ are linked. Otherwise, there is no edge between $p$ and $r$. Because $r \in {lune}_{p s}$, $s \in {lune}_{p t}$, $t \in {lune}_{p u}$, and $u \in {lune}_{p q}$, there are no edges between p and $s,t,u,q$.
Although RNG can reduce its out-degree to a constant $C_{d}+o(1)$, it does not have sufficient edges to be a MSNET. NSG adopts the MRNG's edge selection strategy to construct the index, which is a directed graph. In Figure~\ref{Fig:rng-and-mrng}(b), $p$ and $r$ are linked to each other because there is no node in $lune_{pr}$. $p$ and $s$ are not linked because $p$ and $r$ are linked and $r \in {lune}_{p s}$. However, $p$ and $t$ are linked because $p$ and $s$ are not linked and $s \in {lune}_{p t}$.
The graph constructed by MRNG is an MSNET. The common purpose of these two graph construction methods is to reduce the average out-degree of the graph so as to make the graph sparse and reduce the search complexity. These interesting selection strategies have achieved attractive results, which makes that many graph-based methods perform well in search time, such as Efanna \cite{fu2016efanna}, KGraph, HNSW and NSG.

%They all use different neighbor selection methods to reduce the average out-degree. Reducing the average out-degree can not only reduce the search complexity on the graph, but also reduce the index construction time and index file size. In the high-precision region (usually over 90\%), the search speed can achieve more than 1000 points per second. Besides that, the reduction in average out-degree can also reduce index construction time and index file size.

\begin{figure}[t]
  \centering
  \subfigure[RNG]{\includegraphics[scale=0.3]{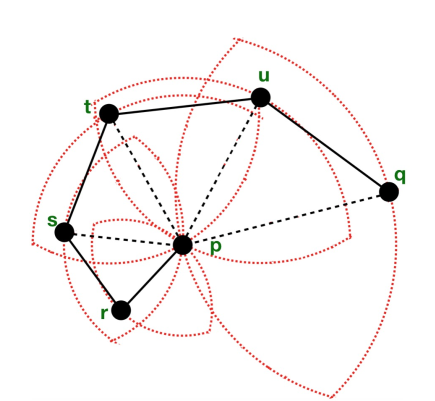}}
  \subfigure[MRNG]{\includegraphics[scale=0.3]{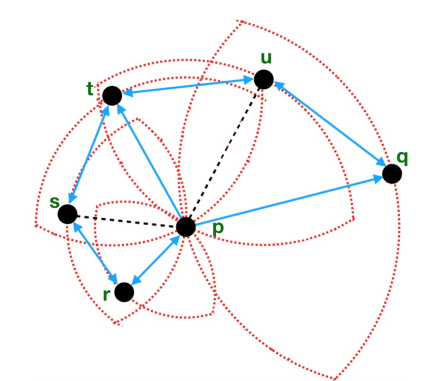}}
  \caption{Two state-of-the-art edge selection strategies}
  \label{Fig:rng-and-mrng}
\end{figure}

\section{Conclusion}
\label{section:conclusion}
This paper takes the first step to shed light on why the graph-based search algorithms work so well in practice and suggests that the clustering coefficient of $K$NN graph is an important measure for the efficiency of these algorithms. Detailed analysis is also conducted to show how clustering coefficient affects the local structure of $K$NN graphs. A few open problems still exists. For example, formal analysis under some simplified data model is important to have more rigorous understanding of the graph search procedure.

\section*{Acknowledgements}
The work reported in this paper is partially supported by NSFC under grant number 61370205, NSF of Xinjiang Key Laboratory under grant number 2019D04024.

%%The authors would like to thank the anonymous reviewers for their valuable comments that help improve the quality of this paper.

%% The Appendices part is started with the command \appendix;
%% appendix sections are then done as normal sections
%% \appendix

%% References with bibTeX database:

\bibliographystyle{IEEEtran}

\bibliography{add}

\end{document}